\documentclass[10pt]{article}

\usepackage[body={6.4in,8.7in}]{geometry}

\usepackage[utf8]{inputenc} 
\usepackage[T1]{fontenc}    
\usepackage{hyperref}       
\usepackage{url}            
\usepackage{booktabs}       
\usepackage{amsfonts}       
\usepackage{nicefrac}       
\usepackage{microtype}      

\usepackage{microtype}
\usepackage{wrapfig}
\usepackage[pdftex]{graphicx}
\usepackage{subfigure}
\usepackage{bm}
\usepackage{comment}
\usepackage[colorinlistoftodos]{todonotes}
\usepackage{amsmath,amssymb,amsthm,mathtools}
\usepackage[ruled,vlined]{algorithm2e}

\title{Polynomial-time Algorithms for Multiple-arm Identification \\ with Full-bandit Feedback}

\author{{\bf Yuko Kuroki}\\
  {\normalsize The University of Tokyo and RIKEN AIP}\\
  {\normalsize \texttt{ykuroki@ms.k.u-tokyo.ac.jp}}
  \and 
  {\bf Liyuan Xu}\\
  {\normalsize The University of Tokyo and RIKEN AIP}\\
  {\normalsize \texttt{liyuan@ms.k.u-tokyo.ac.jp}}
  \and 
  {\bf Atsushi Miyauchi}\\
  {\normalsize RIKEN AIP}\\
  {\normalsize \texttt{atsushi.miyauchi.hv@riken.jp}}
  \and
  {\bf Junya Honda}\\
  {\normalsize The University of Tokyo and RIKEN AIP}\\
  {\normalsize \texttt{honda@stat.t.u-tokyo.ac.jp}}
  \and
  {\bf Masashi Sugiyama }\\
  {\normalsize The University of Tokyo and RIKEN AIP}\\
  {\normalsize \texttt{sugi@k.u-tokyo.ac.jp}}
}
\date{\empty}

\newtheoremstyle{case}{}{}{}{}{}{:}{ }{}
\theoremstyle{case}

\newcommand{\argmax}{\mathop{\rm argmax}}
\newcommand{\argmin}{\mathop{\rm argmin}}
\newcommand{\maxbelow}{\mathop{\rm max}\limits}
\renewcommand{\mid}{\,:\,}
\newcommand{\maxM}{\max_{M \in {\cal M}}}

\newcommand{\cE}{\mathcal{E}}

\newcommand{ \sm   }{ \scalebox{0.8} }

\newcommand{ \chii}{\bm{\chi}_{ \scalebox{0.6} {$M_i$} }}
\newcommand{ \chiM}{\bm{\chi}_{ \scalebox{0.6} {$M$}} }

\newcommand{ \chiMprime}{\bm{\chi}_{ \scalebox{0.6} {$M'_t$}} }

\newcommand{ \chiMstar}{\bm{\chi}_{ \scalebox{0.6} {$M^*$}} }
\newcommand{ \chiMbar}{\bm{\chi}_{ \scalebox{0.6} {$M_i$}} }
\newcommand{ \chiMhatstar}{\bm{\chi}_{\scalebox{0.6} {$\widehat{M}^*_t$}} }

\newcommand{ \chiempbest}{\bm{\chi}_{ \scalebox{0.6} {$\widehat{M}^*_t$}}}

\newcommand{\cM}{ {\cal M}  }
\newcommand{\Atin}{A_{{\bf x}_t}^{-1}}
\newcommand{\At}{A_{{\bf x}_t}}
  
\newcommand{\He}{H_{\varepsilon} }

\newcommand{\bchi}{\bm{\chi}}
\newcommand{\thetahat}{\widehat{\theta}_t}

\newcommand{\Mhatstar}{\widehat{M}^*}
\newcommand{\normAhatstar}{\| \bchi_{\sm {$\Mhatstar$}} \|_{\Atin}}

\newcommand{\normA}{ \|\bchi_{\sm {$M$}} \|_{\Atin}}

\newcommand{\normAstar}{ \|\bchi_{\sm {$M^*$}} \|_{\Atin}}
\newcommand{\normAxy}{ \|\bchi_{\sm {$M$}} -\bchi_{\sm {$M^*$}} \|_{\Atin}}

\theoremstyle{plain}
\newtheorem{theorem}     {Theorem}
\newtheorem{lemma}       {Lemma}
\newtheorem{proposition} {Proposition}
\newtheorem{corollary}   {Corollary}

\begin{document}
\abovedisplayskip=3.5pt
\belowdisplayskip=3.5pt
\setlength\textfloatsep{16pt}

\maketitle

\begin{abstract}
We study the problem of stochastic {\em multiple-arm identification}, where an agent sequentially explores a size-$k$ subset of arms (a.k.a.~a {\em super-arm}) from given $n$ arms and tries to identify the best super-arm. Most existing work has considered the {\em semi-bandit} setting, where the agent can observe the reward of each pulled arm, or assumed each arm can be queried at each round. However, in real-world applications, it is costly or sometimes impossible to observe a reward of individual arms. In this study, we tackle the {\em full-bandit} setting, where only a noisy observation of the total sum of a super-arm is given at each pull. Although our problem can be regarded as an instance of the best arm identification in linear bandits, a naive approach based on linear bandits is computationally infeasible since the number of super-arms $K$ is exponential. To cope with this problem, we first design a polynomial-time approximation algorithm for a 0-1 quadratic programming problem arising in confidence ellipsoid maximization. Based on our approximation algorithm, we propose a bandit algorithm whose computation time is $O(\log K)$, thereby achieving an exponential speedup over linear bandit algorithms. We provide a sample complexity upper bound that is still worst-case optimal. Finally, we conduct experiments on large-scale datasets with more than $10^{10}$ super-arms, demonstrating the superiority of our algorithms in terms of both the computation time and the sample complexity.
\end{abstract}
\section{Introduction}\label{sec:intro}
The stochastic {\em multi-armed bandit (MAB)} is a classical decision making model, 
which characterizes the trade-off between exploration and exploitation in stochastic environments~\cite{Lai1985}.
While the most well-studied objective is to minimize the cumulative regret or maximize the cumulative reward~\cite{ bubeck2012,cesa2006prediction},
another popular objective is to identify the best arm with the maximum expected reward from given $n$ arms.
This problem,
called {\em  pure exploration} or {\em best arm identification} in the MAB,
has received much attention recently~\cite{Audibert2010, Chen2015,Even2002,Even2006,Jamieson2014,kaufmann2016}.

An important variant of the MAB is the {\em multiple-play} MAB problem (MP-MAB),
in which the agent pulls $k\ (\geq 1)$ different arms at each round~\cite{agrawal1990multi,anantharam1987,  Komiyama2015,largre2016}.
In many application domains,
we need to make a decision to
take multiple actions among a set of all possible choices.
For example,
in online advertisement auctions,
companies want to choose multiple keywords to promote their products to consumers based on their search queries~\cite{Rusmevichientong2006}.
From millions of available choices,
a company aims to find the most effective set of keywords by observing the historical performance of the chosen keywords.
This decision making is formulated as the MP-MAB,
where each arm corresponds to each keyword.
In addition, MP-MAB has further applications such as
channel selection in cognitive radio networks~\cite{Huang2008},
ranking web documents~\cite{Radlinski2008},
and crowdsoursing~\cite{Zhou2014}.

In this paper,
we study the {\em multiple-arm identification} 
that corresponds to the pure exploration in the MP-MAB.
In this problem,
the goal is to find the size-$k$ subset (a.k.a.~a super-arm) with the maximum expected rewards.
The problem is also called the {\em  top-$k$ selection}
or {\em $k$-best arm identification},
and has been extensively studied recently \cite{bubeck2013, Cao2015,  Gabillon2012,Gabillon2011,kalyanakrishnan2010,kalyanakrishnan2012, kaufmann2013, Roy2017, Roy2019, Zhou2014}. 
The above prior work has considered the {\em semi-bandit} setting,
in which we can observe a reward of each single-arm in the pulled super-arm,
or assumed that a single-arm can be queried.
However, in many application domains,
it is costly to observe
a reward of individual arms, 
or sometimes we cannot access feedback from individual arms.
For example, in crowdsourcing,
we often obtain a lot of labels given by crowdworkers,
but it is costly to compile labels according to labelers.
Furthermore,
in software projects,
an employer may have complicated tasks that need multiple workers, in which the employer can only evaluate the quality of a completed task rather than a single worker performance~\cite{Retelny2014,Tran2014}.
In such scenarios,
we wish to extract expert workers who can perform the task with high quality, 
only from a sequential access to the quality of the task completed by multiple workers.

In this study,
we tackle the multiple-arm
identification with {\em full-bandit} feedback,
where only a noisy observation of the total sum of a super-arm
is given at each pull rather than a reward of each pulled singe-arm.
This setting is more challenging 
since estimators of expected rewards of single-arms are no longer independent of each other.
We can see our problem as an instance of the pure exploration in \emph{linear bandits}, which has received increasing attention~\cite{Lattimore2017,Soare2014, Tao2018, Xu2018}.
In linear bandits, 
each arm has its own feature $x \in \mathbb{R}^n$,
while in our problem,
each super-arm can be associated with a vector $x \in \{0, 1\}^n$.
Most linear bandit algorithms have, however, the time complexity at least proportional to the number of arms.
Therefore, a naive use of them is computationally infeasible since the number of super-arms $K={ n \choose k}$ is exponential.
A modicum of research on linear bandits addressed the time complexity
~\cite{Kwang2017, Tao2018};
Jun et al.~\cite{Kwang2017} proposed efficient algorithms for regret minimization, which results in the sublinear time complexity $O(K^{\rho})$ for $\rho\in (0,1)$.
Nevertheless, in our setting, they still have to spend $O(n^{\rho k})$ time, where $\rho \in (0,1)$ is a constant, which is exponential.
Thus, to perform multiple-arm identification with full-bandit feedback in practice, the computational infeasibility needs to be overcome
since fast decisions are required in real-world applications.


\paragraph{Our Contribution.}
In this study, we design algorithms,
which are efficient in terms of both the time complexity and the sample complexity.
Our contributions are summarized as follows:

(i)
We propose a polynomial-time approximation algorithm (Algorithm~\ref{alg:QM}) for an NP-hard $0$-$1$ quadratic programming problem arising in confidence ellipsoid maximization.
 In the design of the approximation algorithm,
 we utilize algorithms for a classical combinatorial optimization problem called the {\em densest $k$-subgraph problem (D$k$S)}~\cite{Feige2001}.
 Importantly, we provide a theoretical guarantee for the approximation ratio of our algorithm (Theorem~\ref{thm:approx}).
 
(ii)
Based on our approximation algorithm,
we propose a bandit algorithm (Algorithm~\ref{alg:ECB}) that runs in $O(\log K)$ time (Theorem~\ref{thm:poly}),
and provide an upper bound of the sample complexity (Theorem~\ref{thm:topk}) that is still worst-case optimal. This result means that our algorithm achieves an exponential speedup over linear bandit algorithms while keeping the statistical efficiency.
Moreover,
we propose another algorithm (Algorithm~\ref{alg:ECB2}) that employs the first-order approximation of confidence ellipsoids,
which empirically performs well.


(iii) We conduct a series of experiments on both synthetic and real-world datasets.
First, we run our proposed algorithms on synthetic datasets
and verify that our algorithms give good approximation to an
{\em exhaustive} search algorithm.
Next, we evaluate our algorithms on large-scale crowdsourcing datasets with more than $10^{10}$ super arms, demonstrating the superiority
of our algorithms in terms of both the
time complexity and the sample complexity.

Note that the multiple-arm identification problem
is a special class of the {\em combinatorial pure exploration},
where super-arms follow certain combinatorial constraints such as paths, matchings, or matroids~\cite{cao2017, Chen2016matroid,Chen17a,Chen2014,Gabillon2016,huang2018,pierre2019}.
We can also design a simple algorithm (Algorithm~1 in Appendix~A) for the combinatorial pure exploration under general constraints with full-bandit feedback, which results in looser but general sample complexity bound. For details, see Appendix~A.
Owing to space limitations,
all proofs in this paper are given in Appendix~F.

\section{Preliminaries}\label{sec:prob}




\paragraph{Problem definition.}
Let $[n]=\{1,2,\dots, n\}$ for an integer $n$.
For a vector $x \in \mathbb{R}^n $ and a matrix $B\in \mathbb{R}^{n\times n}$, let $\|x\|_B=\sqrt{x^\top Bx}$. 
For a vector $\theta \in \mathbb{R}^n$
and a subset $S \subseteq [n]$, 
we define $\theta(S)  =\sum_{e\in S} \theta(e)$.  
Now, we describe the problem formulation formally.
Suppose that there are $n$ single-arms associated with unknown reward distributions $\{\phi_1,\dots, \phi_n\}$.
The reward from $\phi_e$ for each single-arm $e \in [n]$ is expressed as $X_t(e)=\theta(e)+ \epsilon_t(e)$, where $\theta(e)$ is the expected reward and $\epsilon_t(e)$ is the zero-mean noise bounded
in $[-R,R]$ for some $R>0$.
The agent chooses a size-$k$ subset
from $n$ single-arms at each round $t$ for an integer $k>0$.
In the well-studied {\em semi-bandit} setting,
the agent pulls a subset $M_t$,
and then she can observe $X_t(e)$ for each $e \in M_t$ independently sampled from the associated unknown distribution $\phi_e$.
However, in the {\em full-bandit} setting,
she can only observe the sum of rewards $r_{M_t}= \theta(M_t)+\sum_{e \in M_t}\epsilon_t(e)$ at each pull,
which means that
estimators of expected rewards of single-arms are no longer independent of each other.

We call a size-$k$ subset of single-arms a {\em super-arm}.
We define a {\em decision class} ${ \cal M}$ as a finite set of super-arms that satisfies the size constraint, i.e.,
${ \cal M}= \{M \subseteq 2^{[n]} \mid |M|=k \}$; thus, the size of the decision class is given by $K = {n \choose k}$.
Let $M^*$ be the optimal super-arm in
the decision class ${\cal M}$,
i.e.,
$M^*={\mathrm{arg\,max}}_{M \in {\cal M}}\theta(M)$.
In this paper,
we focus on the \emph{$(\varepsilon,\delta)$-PAC} setting,
where the goal is to design an algorithm to output the super-arm ${\tt Out} \in {\cal M}$ that satisfies for $\delta \in (0,1)$ and $\varepsilon>0$,
$\Pr[\theta(M^*) - \theta({\tt Out}) \leq \varepsilon ] \geq 1-\delta$.
An algorithm is called {\em $(\varepsilon,\delta)$-PAC} if it satisfies 
this condition.
In the {\em fixed confidence} setting,
the agent's performance is evaluated by her {\em sample complexity}, 
i.e., the number of rounds until the agent terminates.

\paragraph{Technical tools.}

In order to handle full-bandit feedback,
we utilize approaches for best arm identification in
linear bandits. 
Let ${\bf M}_t=(M_1, M_2, \ldots, M_t) \in \cM^t$ be a sequence of super-arms
and $(r_{M_1}, \ldots, r_{M_t}) \in \mathbb{R}^{ t}$ be the corresponding sequence of observed rewards.
Let $\bm{\chi}_{\sm {$M$}} \in \{0,1\}^n$ denote the indicator vector of super-arm $M \in  {\cal M}$;
for each $e \in [n]$, $\chiM(e)=1$ if $e \in M$ and $\chiM(e)=0$ otherwise.
We define the sequence of indicator vectors corresponding to ${\bf M}_t$ as 
${\bf x}_t=(\bm{\chi}_{\sm {$M_1$}},\ldots,\bm{\chi}_{\sm {$M_t$}})$.
An unbiased least-squares estimator for $\theta \in \mathbb{R}^n$
can be obtained by
$\widehat{\theta}_t
=A_{{\bf x}_t }^{-1}b_{{\bf x}_t}\in \mathbb{R}^{n }, 
$
where 
$A_{{\bf x}_t}= \sum_{i=1}^t \chii \chii^{\top}  \in \mathbb{R}^{n \times n}
\ \ \text{and} \ \
b_{{\bf x}_t}=\sum_{i=1}^t\chii r_{M_i} \in \mathbb{R}^n.
$
It suffices to consider the case where  $A_{{\bf x}_t}$ is invertible,
since we shall exclude a redundant feature when any sampling strategy cannot make $A_{{\bf x}_t}$ invertible.
We define the empirical best super-arm as $\Mhatstar_t=\argmax_{M \in \cM}\thetahat(M)$.
\paragraph{Computational hardness.}
The agent continues sampling a super-arm
until a certain stopping condition is satisfied.
In order to check the stopping condition,
existing algorithms for best arm identification in linear bandits involve the following confidence ellipsoid maximization:
\begin{align}\label{prob:CEM}
\text{CEM:}\quad \text{max.} \  \|\chiM \|_{\Atin} \ \ \text{s.t.}\   M \in \cM,
\end{align}
where recall that $\|\chiM \|_{\Atin}= \sqrt{\chiM^{\top} \Atin \chiM}$.
Existing algorithms in linear bandits
implicitly assume that an optimal solution to CEM can be exhaustively searched
(e.g. \cite{Soare2014, Xu2018}).
However,
since the number of super-arms $K$ is exponential in our setting,
it is computationally intractable to exactly solve it.
Therefore,
we need its approximation or
a totally different approach 
for solving the multiple-arm identification with full-bandit feedback.

\section{Confidence Ellipsoid Maximization}\label{sec:topk}
In this section,
we design an approximation algorithm for
confidence ellipsoid maximization CEM.
In the combinatorial optimization literature,
an algorithm is called an \emph{$\alpha$-approximation algorithm}
if it returns a solution that has an objective value greater than or equal to the optimal value times $\alpha \  \in (0, 1]$ for any instance.
Let $W \in \mathbb{R}^{n \times n}$ be a symmetric matrix.
CEM introduced above can be naturally represented by the following 0-1 quadratic programming problem:
\begin{align}\label{prob:QM}
\text{QP:}\quad  \text{max.}\ \sum_{i=1}^n\sum_{j=1}^n w_{ij}x_i x_j
\ \ \ \text{s.t.}\ \ \sum_{i=1}^n x_i=k,\ \ x_i \in \{0,1\}, \ \forall i \in [n].
\end{align}

Notice that QP can be seen as an instance of the {\em uniform quadratic knapsack problem}, which is known to be NP-hard~\cite{Taylor2016}, and there are few results of polynomial-time approximation algorithms even for a special case (see Appendix~C for details).

In this study,
by utilizing algorithms for
a classical combinatorial optimization problem,
called the
{\em densest $k$-subgraph problem (D$k$S)},
we design an approximation algorithm that admits theoretical performance guarantee for QP with positive definite matrix $W$.
The definition of the D$k$S is as follows.
Let $G=(V,E,w)$ be an undirected graph with nonnegative edge weight $w=(w_e)_{e\in E}$. 
  \begin{algorithm}[t]
\caption{Quadratic Maximization}\label{alg:QM}
\SetKwInOut{Input}{Input}
\SetKwInOut{Output}{Output}
\Input{ Symmetric matrix $W\in \mathbb{R}^{n\times n}$}
  
$V \leftarrow [n]$;

$E\leftarrow \{\{i,j\}\mid i,j\in [n],\ i\neq j\}$;
    
\textbf{for} $\{i,j\}\in E$ \textbf{do} $\widetilde{w}_{ij}\leftarrow w_{ij}+w_{ii}+w_{jj}$;

Construct $\widetilde{G}=(V,E,\widetilde{w})$; 

$S \leftarrow $ D$k$S-Oracle$(\widetilde{G})$;

\Return $S$
\end{algorithm}
For a vertex set $S\subseteq V$, let $E(S)=\{\{u,v\}\in E\mid u,v\in S\}$ 
be the subset of edges in the subgraph induced by $S$. 
We denote by $w(S)$ the sum of the edge weights in the subgraph induced by $S$,
i.e., $w(S)=\sum_{e\in E(S)}w_e$.
In the D$k$S, given $G=(V,E,w)$ and positive integer $k$, 
we are asked to find $S\subseteq V$ with $|S|=k$ that maximizes $w(S)$. 
Although the D$k$S is NP-hard,
there are a variety of polynomial-time approximation algorithms~\cite{Asahiro2000,Bhaskara2010,Feige2001}.
The current best approximation result for the D$k$S has an approximation ratio of $\Omega(1/|V|^{{1/4}+\epsilon})$ for any $\epsilon >0$~\cite{Bhaskara2010}. 
The direct reduction of QP to the D$k$S results in an instance that has arbitrary weights of edges.
Existing algorithms cannot be used for such an instance
since these algorithms
need an assumption that
the weights of all edges are nonnegative.


  
    




Now we present our algorithm for QP,
which is detailed in Algorithm~\ref{alg:QM}.
The algorithm operates in two steps. 
In the first step, it constructs an $n$-vertex complete graph $ \widetilde{G}=(V, E, \widetilde{w})$ from a given symmetric matrix $W \in {\mathbb{R}^{n \times n}}$.
 For each $\{i,j\} \in E$,
 the edge weight $\widetilde{w}_{ij}$ is set to $w_{ij}+w_{ii}+w_{jj}$.
Note that if $W$ is positive definite,
$\widetilde{w}_{ij}\geq 0$ holds for every $\{i,j\}\in E$, 
which means that 
$\widetilde{G}$ is an instance of the D$k$S (see Lemma~4 in Appendix~F).
In the second step, the algorithm accesses the {\em densest $k$-subgraph oracle (D$k$S-Oracle)}, 
which accepts $\widetilde{G}$ as input and returns in polynomial time an approximate solution for the D$k$S. 
Note that we can use any polynomial-time approximation algorithm for the D$k$S as the D$k$S-Oracle. 
Let $\alpha_\text{D$k$S}$ be the approximation ratio of the algorithm employed by the D$k$S-Oracle.
By sophisticated analysis on the approximation ratio of Algorithm~\ref{alg:QM},
we have the following theorem.
\begin{theorem}\label{thm:approx}
For QP with any positive definite matrix $W\in \mathbb{R}^{n \times n}$,
Algorithm~\ref{alg:QM} with $\alpha_\text{D$k$S}$-approximation DkS-Oracle is a $\left(\frac{1}{k-1}
\frac{\lambda_{\min}(W)}{\lambda_{\max}(W)}
\alpha_\text{D$k$S}\right)$-approximation algorithm, 
where $\lambda_{\min}(W)$ and $\lambda_{\max}(W)$ represent the minimum and maximum eigenvalues of $W$, respectively.
\end{theorem}
Notice that we prove $\frac{\lambda_{\min}(\At)}{\lambda_{\max}(\At)}=O(1/k)$
for any round $t>n$ in our bandit algorithm (see Lemma~7 in Appendix~F).

  
    





\section{Main Algorithm}\label{sec:topk_bandit_algorithm}

Based on the approximation algorithm proposed in the previous section,
we propose two algorithms for 
the multiple-arm identification with full-bandit feedback.
Note that we assume that $k \geq 2$
since the multiple-arm identification with $k=1$
is the same as best arm identification problem of the MAB.


\paragraph{Static algorithm.}
We deal with {\em static} allocation strategies,
which sequentially sample a super-arm from a fixed sequence of super-arms.
In general, adaptive strategies will perform better than static ones, but due to the computational hardness, we focus on static ones to analyze the worst-case optimality~\cite{Soare2014}.
For static allocation strategies, where ${\bf x}_t$ is fixed beforehand,
Soare et al.~\cite{Soare2014}
provided the following proposition on the confidence ellipsoid for $\widehat{\theta}_t$.

\begin{proposition}[Soare et al.~\cite{Soare2014}, Proposition~1]
Let $\epsilon_t$ be a noise variable
bounded as $\epsilon_t \in [-\sigma, \sigma]$ for $\sigma >0$.
Let $c=2\sqrt{2}\sigma$ and $c'=6/\pi^2$ and fix $\delta \in (0,1)$.
Then, for any fixed sequence ${\bf x}_t$,
with probability at least $1-\delta$,
the inequality
\begin{align}\label{ineq:ellipsoid}
|x^{\top}\theta-x^{\top} \widehat{\theta}_t| \leq C_t \|x\|_{\Atin}
\end{align}
holds for all $t \in \{ 1,2,\ldots\}$
and $x \in \mathbb{R}^n$,
where $C_t =c\sqrt{\log(c't^2 K/\delta) }$.
\end{proposition}
In our problem,
the proposition holds for $\sigma=kR$.
Two allocation strategies named as {\em G-allocation} and {\em ${\cal XY}$-allocation} are discussed in
Soare et al.~\cite{Soare2014}.
Approximating the optimal G-allocation can be done via convex optimization and efficient rounding procedure, and ${\cal XY}$-allocation can be computed in similar manner (see Appendix~D). 
In static algorithms,
the agent pulls a super-arm from a fixed set of super-arms
until a certain stopping condition is satisfied.
Therefore, it is important to construct a stopping 
condition guaranteeing
that the estimate $\thetahat$ belongs to a set of parameters that admits the empirical best super-arm $\widehat{M}^*_t$ as
an optimal super-arm $M^*$ as quickly as possible.
 \begin{algorithm}[t]
  \caption{Static allocation algorithm with approximate quadratic maximization ({\tt SAQM})}\label{alg:ECB}%
	\SetKwInOut{Input}{Input}
	\SetKwInOut{Output}{Output}
	\Input{ Accuracy $\epsilon>0$, confidence level $\delta \in (0,1)$, allocation strategy $p$}
    \For{$t =1, \ldots, n$}{
    $t \leftarrow t+1$ and pull $M_t \in {\rm supp}(p)$;
    
    Observe $r_{M_t}$, and update $A_t$ and $b_t$; 
    }
	\While{stopping condition~\eqref{stop_ecb} is not true}{
    $t \leftarrow t+1$;
    
    Pull $M_t \leftarrow \argmin_{M \in {\rm supp}(p)} \frac{T_M(t)}{p_M}$;
    
    Observe $r_{M_t}$, and update $A_t$ and $b_t$; 
     
    $\widehat{\theta}_t \leftarrow A_{\bf{x}_t}^{-1} b_t$;
    
    $\widehat{M}^*_t   \leftarrow  \argmax_{M \in {\cal M}} \widehat{\theta}_t(M)$;

    $M_t'  \leftarrow $ Quadratic Maximization$(\Atin)$;
    
    $Z_t \leftarrow  C_t\| \chiMprime \|_{\Atin}$;

    }
    
    {\Return $\widehat{M}^* \leftarrow \widehat{M}^*_t $}
\end{algorithm}
\paragraph{Proposed algorithm.}
Now we propose an algorithm named {\tt SAQM}, which is detailed in Algorithm~\ref{alg:ECB}.
Let $\mathcal{P}$ be a $K$-dimensional probability simplex.
We define an allocation strategy $p$ as $p =(p_M)_{M \in {\cal M}}\in {\cal P}$, where $p_M$ prescribes the proportions of pulls to super-arm $M$,
and let ${\rm supp}(p)=\{ M \in {\cal M} \colon p_M>0 \}$ be its support.
Let $T_M(t)$ be the number of times that $M$ is pulled before $(t+1)$-th round.
At each round $t$,
{\tt SAQM} samples a super-arm $M_t=\argmin_{M \in {\rm supp}(p)} T_M(t)/p_M$,
and updates statistics $\At, b_t$ and $\thetahat$.
Then, the algorithm computes the empirical best super-arm $\Mhatstar_t$,
and approximately solves CEM in~\eqref{prob:CEM},
using Algorithm~\ref{alg:QM} as a subroutine.
Note that any $\alpha$-approximation algorithm for QP is 
a $\sqrt{\alpha}$-approximation algorithm for CEM. 
{\tt SAQM} employs the following stopping condition:
\begin{alignat}{4}\label{stop_ecb} 
&\widehat{\theta}_t(\Mhatstar_t)-C_t \|\chiempbest \|_{\Atin} \notag\\
\geq &\max_{M \in {\cal M}\setminus
\{\widehat{M}^*_t\}}
\widehat{\theta}_t(M)+\frac{1}{\alpha_t}{C_tZ_{t}}-\varepsilon,
\end{alignat}
where $Z_t$ denotes the objective value of an approximate solution $M'_t$ for CEM,
and $\alpha_t$ denotes the approximation ratio of our algorithm for CEM at round $t$.
Note that we can compute the value of $\alpha_t$ 
using the guarantee in Theorem~\ref{thm:approx},
and this stopping condition allows the output to be $\varepsilon$-optimal with high probability
(see Lemma~8 in Appendix~F).
As the following theorem states, 
{\tt SAQM} provides an exponential speedup over exhaustive search algorithms.
\begin{theorem}\label{thm:poly}
Let ${\rm poly}(n)_{{\rm D}k {\rm S} }$ be the computation time of the D$k$S-Oracle.
Then, at any round $t>0$,
{\tt SAQM} (Algorithm~\ref{alg:ECB})  runs in $O( \max\{ n^2, {\rm poly}(n)_{{\rm D}k {\rm S} } \})$ time.
\end{theorem}
For example, 
if we employ the algorithm by
Asahiro et al.~\cite{Asahiro2000} as the D$k$S-Oracle in Algorithm~\ref{alg:QM},
the running time becomes $O(n^2)$.
If we employ the algorithm by
Feige, Peleg, and Kortsarz~\cite{Feige2001}, 
the running time of \texttt{SAQM} becomes $O(n^\omega)$,
where the exponent $\omega\le 2.373$ is equal to that of the computation time of matrix multiplication
(e.g., see 
\cite{LeGall14}).


Let $\Lambda_p= \sum_{M \in \cM}p_M\chiM \chiM^{\top}$ be a design matrix.
We define the problem complexity as $H_\varepsilon =\frac{\rho(p)}{(\Delta_{\min}+\varepsilon)^2}$,
where $\rho(p)= \maxM \| \chiM \|_{\Lambda^{-1}_{p}}^2$ and $\Delta_{\min}=\argmin_{M \in \cM \setminus \{ M^*\}} \theta(M^*)-\theta(M)$,
which is also appeared in Soare et al.~\cite{Soare2014}.
The next theorem shows that
{\tt SAQM} is $(\varepsilon ,\delta)$-PAC
and gives a problem-dependent sample complexity bound.
\begin{theorem} \label{thm:topk}
Given any instance of the multiple-arm identification with full-bandit feedback,
with probability at least $1-\delta$, 
{\tt SAQM} (Algorithm~\ref{alg:ECB}) returns an $\varepsilon$-optimal super-arm $\widehat{M}^*$
and the total number of samples $T$ is bounded as follows:
\begin{align*}
T =
O\left(
k^2\He \left(
n^{\frac{1}{4}}k^3 \log \left( \frac{n}{\delta}\right)
+ \log \left( 
n^{\frac{1}{8}} k^3 \He 
\left(
n^{\frac{1}{4}} k^3\He+
\log\left( \frac{n}{\delta}  \right)
\right)
\right)
\right)
\right).
\end{align*}



\end{theorem} 
It is worth mentioning that if we have an $\alpha$-approximation algorithm for CEM with a more general decision class ${\cal M}$ (such as paths, matchings, matroids),
we can extend
Theorem~\ref{thm:topk} for the combinatorial pure exploration (CPE) with general constraints
as follows. 
\begin{corollary}\label{col:general}
Given any instance of CPE with a decision class ${\cal M}$ in the full-bandit setting,
with probability at least $1- \delta$, 
{\tt SAQM} (Algorithm~\ref{alg:ECB})
with $\alpha$-approximation of CEM
returns
an
$\varepsilon$-optimal set $\widehat{M}^*$,
and the total number of samples $T$ is bounded as follows:
\[
T\leq  8\left(3+\frac{1}{\alpha}\right)^2 k^2 \He \log\left (   \frac{c'K}{\delta} \right)+C(H_\varepsilon, \delta),
\]
where 
$C(\He, \delta)=O\left(k^2 \He \log \left(  \frac{k}{\alpha} \He \left(   \frac{k^2}{\alpha^2}\He +\log\left( \frac{K}{\delta}\right) \right)  \right)\right)$.
\end{corollary} 
Theorem~\ref{thm:topk} corresponds to the case where $\alpha=O(1/kn^{\frac{1}{8}})$ in Corollary~\ref{col:general}.
Soare et al.~\cite{Soare2014} considered the {\em oracle sample complexity} of a linear best-arm identification problem.
The oracle complexity,
which is based on the optimal allocation strategy $p$ derived from the true parameter $\theta$,
is $O(\rho(p)\log(1/\delta))$.
Soare et al.~\cite{Soare2014}
showed that the sample complexity with G-allocation strategy matches 
the oracle sample complexity 
up to constants in the worst case.
The sample complexity of {\tt SAQM} is also worst-case optimal in the sense that it matches $O(\rho(p)\log(1/\delta))$,
while {\tt SAQM} runs in polynomial time.
\begin{algorithm}[t]
\caption{Static allocation algorithm with first-order approximation ({\tt SA-FOA)}}

	\label{alg:ECB2}
	\SetKwInOut{Input}{Input}
	\SetKwInOut{Output}{Output}
	\Input{ Accuracy $\epsilon>0$, confidence level $\delta \in (0,1)$, allocation strategy $p$}
    \For{$t =1, \ldots, n$}{
    $t \leftarrow t+1$ and pull $M_t \in {\rm supp}(p)$;
    
    Observe $r_{M_t}$, and update $A_t$ and $b_t$; 
    }
	\While{$\frac{\varepsilon}{2} \geq Z'_t-\widehat{\theta}_t(\Mhatstar_t)$  is not true}{
    $t \leftarrow t+1$;
    
    Pull $M_t \leftarrow \argmin_{M \in {\rm supp}(p)} \frac{T_M(t)}{p_M}$;
    
    Observe $r_{M_t}$, and then update $A_t$, $b_t$ and $\widehat{\theta}_t \leftarrow A_{\bf{x}_t}^{-1} b_t$; 
    
    
    $\widehat{M}^*_t   \leftarrow  \argmax_{M \in {\cal M}} \widehat{\theta}_t(M)$;

    $m \leftarrow \ell n$ for some positive integer $\ell$; 
    
    \For{$i=1, \ldots,m$}{
    $F \leftarrow \{\emptyset\}$;
    
    
    Choose a super arm $M_i \in {\rm supp}(p)$ at uniformly random;
    
    $\gamma \leftarrow \frac{C_t}{2 \|\chiMbar-\chiMhatstar  \|_{\Atin}}$;
    
    $B_t \leftarrow \gamma\Atin-{\rm Diag}(2\gamma(\Atin\chiMhatstar))+{\rm Diag}(\thetahat)$;
    


    
    $F \leftarrow F \cup \{  {\rm Quadratic Maximization}(B_t)\}$;

    }
    

    $Z'_t \leftarrow \underset{M \in  F}{\max}\left( \thetahat(M) +C_t\|\chiM-\chiMhatstar\|_{\Atin}\right)$;
}
    
    \Return{$\widehat{M}^*\leftarrow \widehat{M}^*_t$}
\end{algorithm}

\paragraph{Heuristic algorithm.}
In {\tt SAQM},
we compute an upper confidence bound of the expected reward of each super-arm.
However,
in order to reduce the number of required samples,
we wish to directly construct
a tight confidence bound
for the gap of the reward between two super-arms.
For this reason,
we propose another algorithm {\tt SA-FOA}.
The procedure of {\tt SA-FOA} is shown in Algorithm~\ref{alg:ECB2}.
Given an allocation strategy $p$,
this algorithm continues sampling
until the stopping condition $\frac{\varepsilon}{2} \geq Z'_t-\widehat{\theta}_t(\Mhatstar_t)$ is satisfied,
where $Z'_t$ denotes the objective value of an approximate solution of the following maximization problem:
\begin{align}\label{prob:FOA}
\max_{M\in { \cal M}\setminus \{\widehat{M}^*_t\}} \left(\widehat{\theta}_t(M)+C_t \|\chiM-\chiMhatstar \|_{\Atin} \right).
\end{align}
The second term of \eqref{prob:FOA} can be regarded as the
confidence interval of the estimated gap $\thetahat(M)-\thetahat(\Mhatstar_t)$.
We employ a first-order approximation technique,
in order to simultaneously maximize the estimated reward  $\widehat{\theta}_t(M)$
and the matrix norm $\|\chiM-\chiMhatstar \|_{\Atin}$. 
For a fixed super-arm $M_i$,
we approximate $\|\chiM-\chiMhatstar\|_{\Atin}$ using the following bound: 
\begin{alignat*}{4}
&\|\chiM-\chiMhatstar\|_{\Atin}  
&\leq \frac{\|\chiM-\chiMhatstar\|_{\Atin}^2 }{2\|\chiMbar-\chiMhatstar\|_{\Atin} }
+\frac{\|\chiMbar-\chiMhatstar\|_{\Atin}}{2},
\end{alignat*}
which follows from $\sqrt{a+x} \leq \sqrt{a}+\frac{x}{2\sqrt{a}}$ for any $a,x> 0$.
For any $y \in \mathbb{R}^n$, let ${\rm Diag}(y)$ be a matrix whose  $i$-th diagonal component is $y(i)$ for $i \in [n]$.
The above first-order approximation allows us
to transform
the original problem
to QP,
where the objective function is
\[
\chiM^{\top}\left(\gamma\Atin-{\rm Diag}(2\gamma(\Atin\chiMhatstar))+{\rm Diag}(\thetahat) \right)\chiM, \ \text{with a positive constant } \gamma=\frac{C_t}{2\|\chiMbar-\chiMhatstar \|_{\Atin}}.
\] 
We can approximately solve it by Algorithm~\ref{alg:QM},
and choose the best approximate solution
that maximizes the original objective among $\ell n$ super arms.
Notice that {\tt SA-FOA} is an 
$(\epsilon, \delta)$-PAC algorithm
since we compute the upper bound of the objective function in~\eqref{prob:FOA} and thus it will not stop earlier.
In our experiments,
it works well although we have no theoretical results on the sample complexity.
We will also observe in the experiments that
the approximation error of {\tt SA-FOA} for \eqref{prob:FOA}
becomes smaller as the number of rounds increases.

\section{Experiments}\label{sec:experiments}

\begin{figure}[t!]
 \begin{minipage}{0.5\hsize}
  \begin{center}
\scalebox{0.95}{
   \includegraphics[width=70mm]{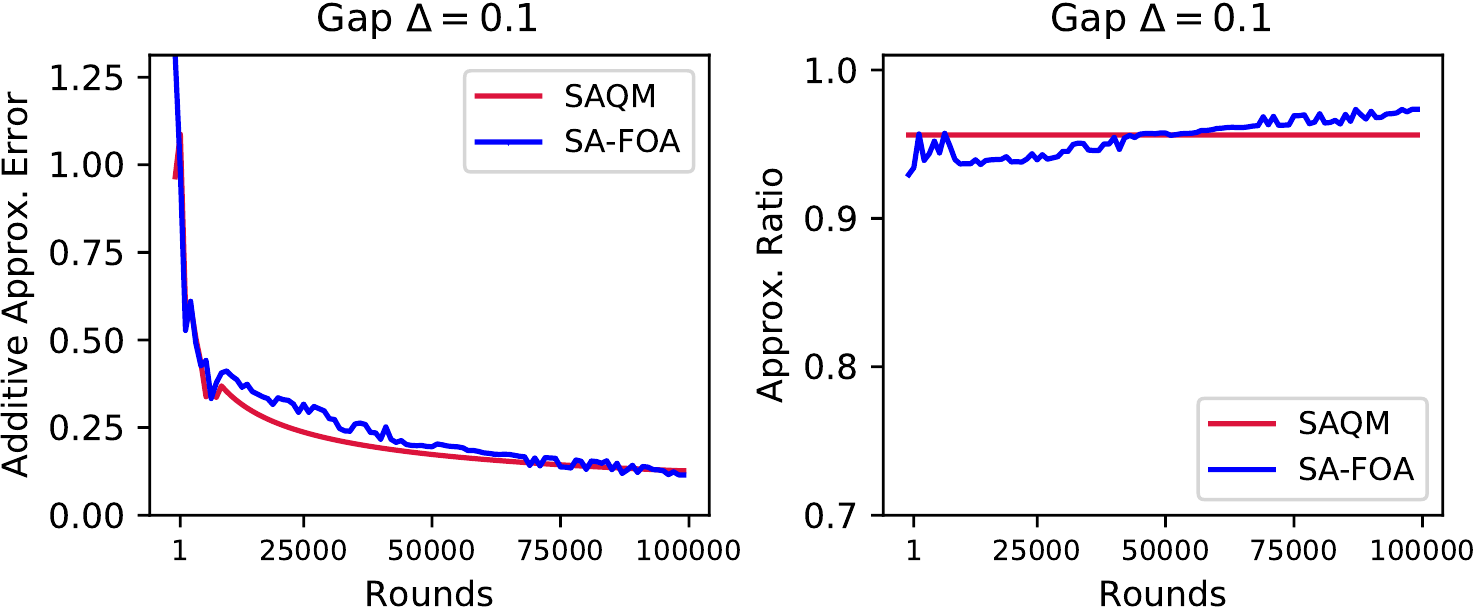}
   }
  \end{center}
  \label{fig:one}
 \end{minipage}
 \begin{minipage}{0.5\hsize}
  \begin{center}
\scalebox{0.95}{
   \includegraphics[width=70mm]{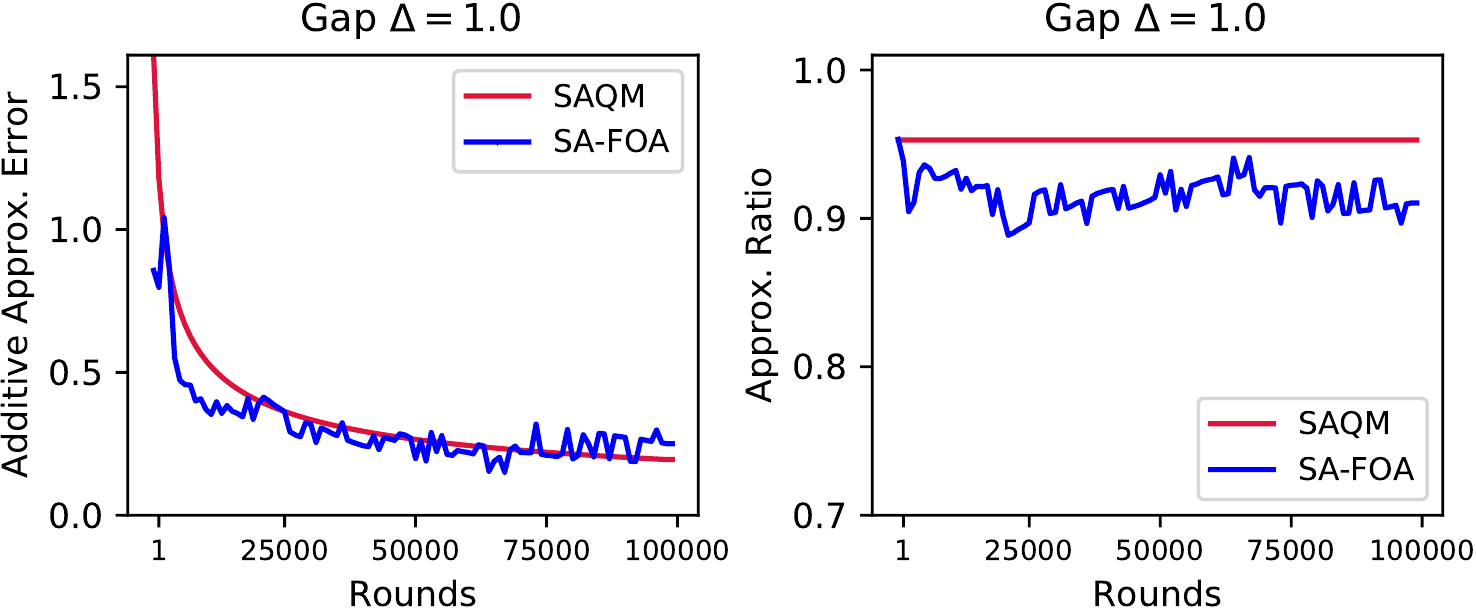}
   }
  \end{center}
  \label{fig:two}
 \end{minipage}
  \caption{Approximation precision for synthetic datasets with $(n,k)=(10,5)$. Each point corresponds to an average over 10 realizations.}\label{fig:approx}
\end{figure}

\begin{figure}[t!]

 \begin{minipage}{0.48\hsize}
  \begin{center}
 \scalebox{0.90}{
   \includegraphics[width=48mm]{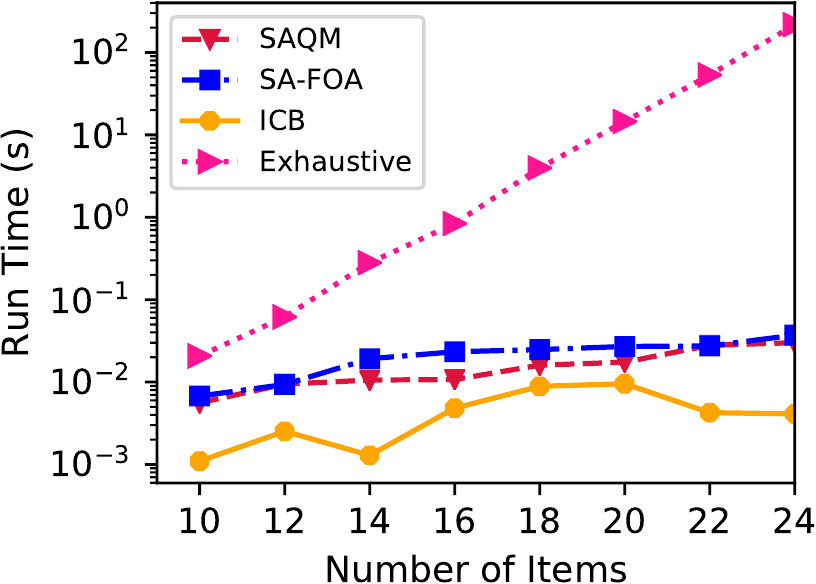}
  }
  \end{center}
    \caption{Run time in each round for synthetic datasets.
   Each point is an average over 10 realizations.}
  \label{fig:runtime}
 \end{minipage}
  \hfill
  \begin{minipage}{0.48\hsize}
  \begin{center}
 \scalebox{0.90}{
   \includegraphics[width=48mm]{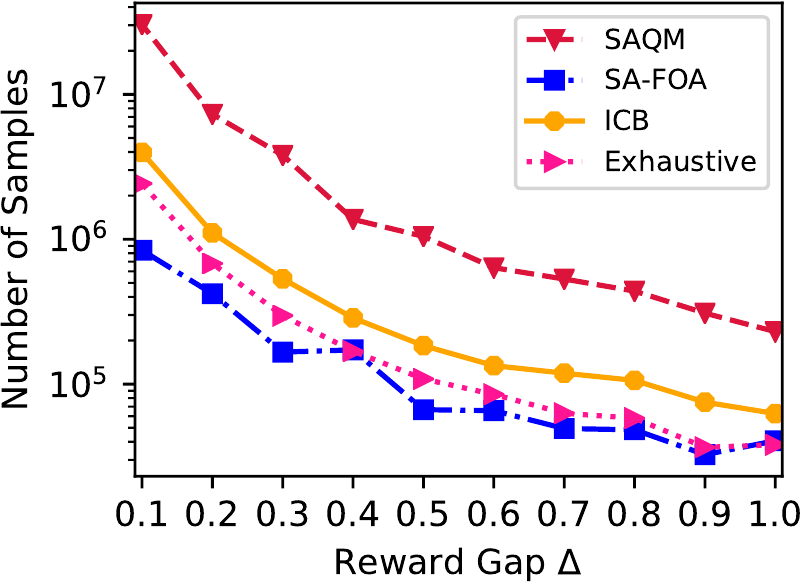}
  }
  \end{center}
    \caption{Number of samples for synthetic datasets with $(n,k)=(10,5)$.
   Each point is an average over 10 realizations.}
\label{fig:sample}
 \end{minipage}
\end{figure}


In this section, we evaluate the empirical performance of our algorithms, namely {\tt SAQM} (Algorithm~\ref{alg:ECB}) and
{\tt SA-FOA} (Algorithm~\ref{alg:ECB2}).
We also implement another algorithm namely {\tt ICB} (Algorithm~4 in Appendix~A) as a naive algorithm that works in polynomial-time. {\tt ICB} employs simplified confidence bounds obtained by diagonal approximation of confidence ellipsoids. 
Note that {\tt ICB} can solve the combinatorial pure exploration problem with general constraints and results in another sample complexity (see Lemma~3 in Appendix~A).
We compare our algorithms with an exhaustive search algorithm namely {\tt Exhaustive}, which runs in exponential time
(see Appendix~E for details).
We conduct the experiments on small synthetic datasets and large-scale real-world datasets.

\paragraph{Synthetic datasets.}


To see the dependence of the performance on the minimum gap $\Delta_{\min}$,
we generate synthetic instances as follows.
We first set the expected rewards
for the top-$k$ single-arms uniformly at random from $[0,1]$.
Let $\theta_{{\min}\text{-}k}$
be the the minimum expected reward
in the top-$k$ single-arms.
We set the expected reward of
the $(k+1)$-th best single-arm to
$\theta_{{\min}\text{-}k}-\Delta_{\min}$
for the predetermined parameter $\Delta_{\min} \in [0, 1]$.
Then, 
we generate the expected rewards of
the rest of single-arms by uniform samples from $[-1, \theta_{{\min}\text{-}k}-\Delta_{\min}]$
so that expected rewards of the best super-arm is
larger than those of the rest of super-arms by at least $\Delta_{\min}$.
We set the additive noise distribution ${\cal N}(0,1)$
and $\delta=0.05$.
All algorithms employ G-allocation strategy.

First,
we examine the approximation precision of our approximation algorithms.
The results are reported in Figure~\ref{fig:approx}.
{\tt SAQM} and {\tt SA-FOA} employ some approximation mechanisms to test the stopping condition in polynomial time.
Recall that {\tt SAQM} approximately solves CEM in~\eqref{prob:CEM}
to attain an objective value of $Z_t$,
and 
{\tt SA-FOA} approximately solves the maximization problem in~\eqref{prob:FOA}
to attain an objective value of $Z'_t$.
We set up the experiments with $n=10$ single-arms and $k=5$.
We run the experiments for the small gap ($\Delta_{\min} = 0.1$)
and large gap ($\Delta_{\min} = 1.0$).
We plot the approximation ratio and the additive approximation error of {\tt SAQM} and {\tt SA-FOA} in the first 100,000 rounds.
From the results,
we can see that the approximation ratios
of them are almost always greater than $0.9$, 
which are far better than the worst-case guarantee proved in Theorem~\ref{thm:approx}.
In particular,
the approximation ratio of {\tt SA-FOA}
in the small gap case
is surprisingly good (around 0.95)
and grows as the number of rounds increases.
This result implies that
there is only a slight
increase of the sample complexity
caused by the approximation,
especially
when the expected rewards of single-arms
are close to each other.

\begin{table}[t!]
  \begin{minipage}[t]{.50\textwidth}
    \caption{Real-world datasets on crowdsourcing. ``Average'' and ``Best'' give the average and the best accuracy rate among the workers, respectively.}\label{tab:instance}
    \begin{center}
\scalebox{0.80}{
\begin{tabular}{l*{4}{@{\hspace{3mm}}r}}
\toprule
Dataset     & $\#$task & $\#$worker &  \text{Average} &  \text{Best}\\
\midrule
{\em IT         } & 25        & 36        & 0.54   & 0.84 \\
{\em Medicine   } & 36        & 45        & 0.48  & 0.92 \\
{\em Chinese    } & 24        & 50        & 0.37   & 0.79 \\
{\em Pok\'{e}mon} & 20        & 55        & 0.28  & 1.00 \\
{\em English    } & 30        & 63        & 0.26   & 0.70 \\
{\em Science    } & 20        & 111       & 0.29  & 0.85 \\ 
\bottomrule
\end{tabular}   
}
   \end{center}
  \end{minipage}
  \hspace{1.0cm}
  \begin{minipage}[t]{.41\textwidth}
\caption{Number of samples ($\times 10^3$) on real-world crowdsourcing datasets (average over 5 realizations).}\label{fig:realdata}
    \begin{center}
    \scalebox{0.80}{
\begin{tabular}{l*{3}{@{\hspace{3mm}}r}}
\toprule
Dataset    & \texttt{ICB} & \texttt{SAQM} &  \texttt{SA-FOA} \\
\midrule
{\em IT         } &  46,658    & 68,328    & 3,421\\
{\em Medicine   } &  73,337     & 86,252    & 3,493 \\
{\em Chinese    } &  105,214     & 110,504     & 4,949\\
{\em Pok\'{e}mon} &  20,943   & 91,423    & 3,050\\
{\em English    } &  118,587     & 131,512    & 9,313  \\
{\em Science    } & 362,558     & 291,773     & 15,611 \\ 
\bottomrule
\end{tabular}  
}
   \end{center}
  \end{minipage}
\end{table}

Next, we conduct the experiments to compare the running time of algorithms.
We set $n=10,12,\ldots,24$ and $k=n/2$ on synthetic datasets.
We report the results in Figure~\ref{fig:runtime}.
As can be seen,
{\tt Exhaustive} is prohibitive on instances with large number of super-arms,
while our algorithms can run fast even if $n$ becomes larger,
which matches our theoretical analysis.
The results indicate that polynomial-time algorithms are of crucial importance for practical use.

Finally, we evaluate the number of samples required to identify the best super-arm
for varying $\Delta_{\min}$.
Based on the above observation, we set $\alpha=0.9$.
The result is shown in Figure~\ref{fig:sample},
which 
indicates that 
the numbers of samples of our algorithms are comparable to that of {\tt Exhaustive}.
We observed that our algorithms always output the optimal super-arm.






\paragraph{Real-world datasets on crowdsourcing.}

We use the crowdsourcing datasets
compiled by 
Li et al.~\cite{Li2017}
whose basic information is shown in Table~\ref{tab:instance}.
The task is to identify the top-$k$ workers with the highest accuracy
only from a sequential access to the accuracy of part of labels given by some workers.
Notice that
the number of super-arms is more than $10^{10}$ and $ \Delta_{\min}$ is less than $0.1$ in all experiments.
We set 
$k=10$ and  $\varepsilon=0.5$.
Since {\tt Exhaustive} is prohibitive,
we compare other three algorithms.
All algorithms employ uniform allocation strategy.
The result is shown in
Table
\ref{fig:realdata},
which indicates
the applicability of our algorithms to the instances with a massive number of super-arms.
Moreover,
all three algorithms found the optimal subset of crowdworkers.
In all datasets, {\tt SA-FOA} outperformed the other algorithms.
{\tt ICB} also worked well,
but it became worse especially for {\em Science} in which the number of workers is more than $100$.
This result implies that when the number of workers (single-arms) is large,
the algorithm with simplified confidence bounds may degrate the sample complexity,
while algorithms with confidence ellipsoids require less samples as {\tt SAQM} and {\tt SA-FOA} perform well
(see Appendix~A for more discussion).


\section{Conclusion}
We studied the multiple-arm identification with full-bandit feedback,
where we cannot observe a reward of each single-arm, but only the sum of the rewards.
To overcome the computational challenges,
we designed a novel approximation algorithm
for a 0-1 quadratic programming problem with theoretical guarantee. 
Based on our approximation algorithm,
we proposed a polynomial-time algorithm {\tt SAQM} that runs in $O(\log K)$ time
and provided an upper bound of the sample complexity,
which is still worst-case optimal;
the result indicates that our algorithm provided an exponential speedup over exhaustive search algorithm while keeping the statistical efficiency.
We also designed a novel algorithm {\tt SA-FOA} using first-order approximation that empirically performs well.
Finally,
we conducted experiments on synthetic and real-world datasets with more than $10^{10}$ super-arms, demonstrating the superiority of our algorithms in terms of both the computation time and the sample complexity.
There are several directions for future research.
It remains open to design adaptive algorithms with a problem-dependent optimal sample complexity. It is also interesting question to seek a lower bound of any $(\epsilon, \delta)$-PAC algorithm that works in polynomial-time. Extension for combinatorial pure exploration with full-bandit feedback is another direction.

\clearpage
\medskip

\bibliographystyle{abbrv}
\bibliography{mybib}

\appendix
\setcounter{theorem}{1}
\setcounter{lemma}{0}
\setcounter{corollary}{0}
 \setcounter{equation}{0}
\setcounter{proposition}{0}
\abovedisplayskip=7pt
\belowdisplayskip=7pt

\section{Simplified Confidence Bounds for the Combinatorial Pure Exploration}\label{apx:cpe}

In this appendix,
we see the fundamental observation of employing a simplified confidence bound
to obtain a computational efficient algorithm for the combinatorial pure exploration problem.
We consider any decision class ${\cal M}$,
in which super-arms satisfy any constraint where a linear maximization problem is polynomial-time solvable.
The examples of decision class considered here are paths, matchings, or matroids (see Appendix~B for the definition of matroids).
The purpose of this appendix is to give a polynomial-time algorithm for solving the combinatorial pure exploration with general constraints
by using the simplified confidence bound,
and see the trade-off
between the statistical efficiency and computational efficiency.
The $(\epsilon, \delta)$-PAC algorithm proposed in this section, named {\tt ICB},
is also evaluated in our experiments.

For a matrix $B \in \mathbb{R}^{n \times n}$, let $B(i,j)$ denote the $(i,j)$-th entry of $B$. We construct a simplified confidence bound, named a {\em independent confidence bound},
which is obtained by diagonal approximation of confidence ellipsoids.
We start with the following lemma,
which shows that
$\theta$ lies 
in an independent confidence region
centered at $\thetahat$ with high-probability.
\begin{lemma}\label{thm:independent}
Let  $c'=6/\pi^2$.
Let $\epsilon_t$ be a noise variable
bounded as $\epsilon_t \in [-\sigma, \sigma]$ for $\sigma >0$.
Then,
for any fixed sequence ${\bf x}_t$,
any $t \in \{ 1,2,\ldots\}$,
and $\delta \in (0, 1)$,
with probability at least $1-\delta$,
the inequality
\begin{align}\label{ineq:independent}
|x^{\top}\theta-x^{\top} \widehat{\theta}_t| \leq C_t \sum_{i=1}^n |x_i|\sqrt{\Atin(i,i)}
\end{align}
holds for all $x \in \{-1, 0, 1\}^n$, 
where
\begin{align*}
C_t =\sigma\sqrt{2\log(c't^2 n/\delta) }.
\end{align*}
\end{lemma}
This lemma can be derived from Proposition~1 and
the triangle inequality.
The RHS of (\ref{ineq:independent})
only has linear terms of $\{x_i\}_{i \in [n]}$,
whereas
that of (3) in Proposition~1
has the matrix norm $\|x\|_{\Atin}$, which results in
a difficult instance. 
As long as we assume that linear maximization oracle is available,
maximization of this value can be also done in polynomial time.
For example,
maximization of the RHS of (\ref{ineq:independent})
under matroid constraints
can be solved by using the simple greedy procedure~\cite{Karger1998} described in Appendix~B.
Based on the independent confidence bounds,
we propose {\tt ICB}, 
which is detailed in Algorithm~\ref{alg:ICB}.
At each round $t$,
{\tt  ICB} computes the empirical best super-arm $\widehat{M}^*_t$, 
and then solves the following maximization problem:
\begin{alignat*}{3}
&\mathrm{P}_1: &\ \ &\text{max.}&\ &\widehat{\theta}_t( M)+C_t \sum_{i=1}^n |\chiM(i)-\chiMhatstar(i)| \sqrt{\Atin(i,i)},\\
&             &    &\text{s.t.}&  &M\in \mathcal{M}\setminus \{\widehat{M}^*_t\}.
\end{alignat*}
The second term in the objective of $\mathrm{P}_1$
can be regarded
as the confidence interval
of the estimated
gap $\widehat{\theta}_t(M)-
\widehat{\theta}_t(\widehat{M}^*_t)$.
 {\tt  ICB} continues
 sampling a super-arm until the following stopping condition is satisfied:
\begin{align}\label{stop_independet}
Z_t^* -\widehat{\theta}_t(\widehat{M}_t^*)< \varepsilon,
\end{align}
where $Z_t^*$ represents the optimal value of $\mathrm{P}_1$. 
Note that
$\mathrm{P}_1$ is solvable in polynomial time
because $\mathrm{P}_1$ is an instance of linear maximization problems.
As the following lemma states, {\tt ICB} is an efficient algorithm in terms of the computation time.
\begin{lemma}\label{thm:matroid_time}
Given any instance of combinatorial pure exploration with full-bandit feedback with decision class $\cM$,
{\tt ICB} (Algorithm~\ref{alg:ICB})
at each round $t  \in \{1,2,\ldots\}$
runs in polynomial time.
\end{lemma}
The proof is given in Appendix~F.
For example,
{\tt ICB} runs in $O(\max\{n^2, ng(n) \})$ time for matroid constraints,
where $g(n)$ is the computation time to check whether given super-arm
is contained in the decision class.
Note that $g(n)$ is polynomial in $n$ for any matroid constraints. 
For example, $g(n) = O(n)$ if we consider the case where 
each super-arm corresponds to
a spanning tree of a graph $G=(V,E)$, 
and a decision class corresponds to a set of spanning trees in a given graph $G$.

\begin{algorithm}[t]
	\caption{Static allocation with independent confidence bound (\texttt{ICB})}\label{alg:ICB}
	\SetKwInOut{Input}{Input}
	\SetKwInOut{Output}{Output}
	\Input{Accuracy $\epsilon>0$,
	confidence level $\delta \in (0,1)$,
	allocation strategy $p$}
    \For{$t =1, \ldots, n$}{
    
    $t \leftarrow t+1$;
    
    Pull $M_t \in {\rm supp}(p)$;
    
    Observe $r_t$;
    
    Update $A_t$ and $b_t$;
    }
	\While{stopping condition \eqref{stop_independet} is not true}{ 
	
    $t \leftarrow t+1$;
    
    Pull $M_t \leftarrow \argmin_{M \in {\rm supp}(p)} \frac{T_M(t)}{p_M}$;

    Observe $r_t$;
    
    Update $A_t$ and $b_t$;
     
    $\widehat{\theta}_t \leftarrow A_{\bf{x}_t}^{-1} b_t$;  
    
    $\widehat{M}^*_t   \leftarrow  \argmax_{M \in {\cal M}} \widehat{\theta}_t(M)$;
    
    $Z_t^* \leftarrow \maxbelow_{M \in \mathcal{M}\setminus \{\widehat{M}^*_t\}} \left( \widehat{\theta}_t(M)+
    C_t\sum_{i=1}^n |(\chiM(i)-\chiMhatstar(i))| 
    \sqrt{\Atin(i,i)} \right)$;

    }
    
    \Return{$\widehat{M}^*\leftarrow \widehat{M}^*_t$}
\end{algorithm}

From the definition,
we have $A_t=\sum_{M \in {\cal M}}T_M(t)\chiM \chiM^{\top}$,
where $T_M(t)$ denotes the number of times that $M$ is pulled before
the round $t+1$.
Let  $\Lambda'_{p}=
\sum_{M \in \cM} p_M \chiM \chiM^{\top}$.
We define $\rho'(p)$ as
\begin{align}\label{rhop}
 \rho'(p)= \left(\max_{M, M' \in \cM} \sum_{i=1}^n
 |\chiM(i)-\bm{\chi}_{ \scalebox{0.5} {$M'$}} (i)| \sqrt{{\Lambda^{-1}_{p}}(i,i)}\right)^2.
\end{align}
Now, we give a problem-dependent sample complexity bound of {\tt ICB}
with allocation strategy $p$
as follows.
\begin{lemma}\label{thm:matroid}
Given any instance of combinatorial pure exploration with decision class $\cM$ in full-bandit setting,
with probability at least $1- \delta$, 
{\tt ICB} (Algorithm~\ref{alg:ICB}) returns an $\varepsilon$-optimal super-arm $\widehat{M}^*$ and the total  number of samples $T$ is bounded as follows:
\begin{align*}
T =
O\left(
k^2\He' \log \left( 
\frac{n}{\delta} 
\left(
k\He'
\left(k^2\He' +\log\left( \frac{n}{\delta}  \right)\right)
\right)
\right)
\right), \text{where } H'_\epsilon= 
\frac{\rho'(p)}{(\Delta_{\min}+\varepsilon)^2}.
\end{align*}
\end{lemma}
The proof is given in Appendix~F.
Notice that in the MAB, this diagonal approximation is tight since  $\At$ becomes a diagonal matrix.
However, for combinatorial settings where the size of super-arms is $k \geq 2$,
there is no guarantee that this approximation is tight;
the approximation may degrate the sample complexity.
Although the proposed algorithm here empirically perform well when the number of single-arms is not large,
it is still unclear that using the simplified confidence bound should be desired instead of ellipsoids confidence bounds since $\rho'(p)$ is $\Omega(n)$.
This is the reason why we focus on the approach with confidence ellipsoids.

\section{Definition of Matroids}\label{appendix:matroid}
A {\em matroid} is a combinatorial structure that abstracts many notions of independence
such as linearly independent vectors in a set of vectors (called the {\em linear matroid}) 
and spanning trees in a graph (called the {\em graphical matroid}) \cite{whitney1935}.
Formally, a matroid is a pair $J=(E, {\cal I})$,
where $E=\{1,2,\ldots,n\}$ is a finite set called a {\em ground set} and ${\cal I}\subseteq 2^{E}$ is a family of subsets of $E$ called {\em independent sets}, that satisfies the following axioms: 
\begin{enumerate}
\item $\emptyset \in {\cal I}$;
\item $X\subseteq Y\in \mathcal{I}\Longrightarrow X\in \mathcal{I}$; 
\item $\forall X, Y \in {\cal I}$ such that $|X|<|Y|$, $\exists e\in Y\setminus X$ such that $X\cup \{e\}\in \mathcal{I}$. 
\end{enumerate}
A \emph{weighted matroid} is a matroid that has a weight function $w\colon E\rightarrow \mathbb{R}$. 
For $F\subseteq E$, we define the weight of $F$ as $w(F)=\sum_{e\in F}w(e)$. 

Let us consider the following problem: given a weighted matroid $J=(E,\mathcal{I})$ with $w\colon E\rightarrow \mathbb{R}$, we are asked to find an independent set with the maximum weight, i.e., $\argmax_{F\in \mathcal{I}}w(F)$. 
This problem can be solved exactly by the following simple greedy algorithm~\cite{Karger1998}. 
The algorithm initially sets $F$ to the empty set. 
Then, the algorithm sorts the elements in $E$ with the decreasing order by weight, 
and for each element $e$ in this order, the algorithm adds $e$ to $F$ if $F\cup \{e\}\in \mathcal{I}$. 
Letting $g(n)$ be the computation time for checking whether $F$ is independent, 
we see that the running time of the above algorithm is $O(n\log n+ng(n))$.

\section{Uniform Quadratic Knapsack Problem}
\setcounter{equation}{0}
Assume that we have $n$ items,
each of which has  weight $1$.
In addition,
we are given an $n \times n$ non-negative integer matrix $W = (w_{ij})$,
where $w_{ii}$ is the profit achieved
if item $i$ is selected
and $w_{ij} + w_{ji}$ is
a profit achieved
if both items $ i$ and $j$ are selected for $i<j$.
The uniform quadratic knapsack problem (UQKP) calls for selecting a subset of items whose overall weight does not exceed a given knapsack capacity $k$,
 so as to maximize the overall profit.
The UQKP can be formulated as the following $0$-$1$ integer quadratic programming:
\begin{alignat*}{4}
&\text{max.} &\ \   &\sum_{i=1}^n\sum_{j=1}^n w_{ij}x_i x_j \\\notag
&\text{s.t.} &      &\sum_{i=1}^n x_i \leq k.
\end{alignat*}
The UQKP is an NP-hard problem.
Indeed,
the maximum clique problem, which is also NP-hard,
can be reduced to it;
Given a graph $G=(V,E)$,
we set $w_{ii}=0$ for all $i$ and $w_{ij}=1$ for all $\{i,j\} \in E$.
Solving this problem,
it allows us to find a clique of size $k$
if and only if the optimal solution of the problem has value $k(k-1)$~\cite{Taylor2016}.

\section{Allocation Strategies}
\setcounter{equation}{0}
In this section,
we briefly introduce the possible allocation strategies and
describe how to implement a continuous allocation $p$ into a discrete allocation ${\bf x}_t$ for any sample size $t$.
We report the efficient rounding procedure introduced
in~\cite{pukelsheim2006}.
In the {\em G-allocation} strategy,
we make the sequence of selection $\bf{x}_t$ to be
${\bf x}^G_{t}= \argmin_{{\bf x}_t \in \mathbb{R}^{n \times t}}
 \max_{x \in {\cal X}} \|x \|_{\Atin}$ for ${\cal X} \subseteq \mathbb{R}^n$, which is NP-hard optimization problem.
 There are massive studies that proposed approximate solutions
 to solve it in the experimental design literature ~\cite{mustapha2010, SAGNOL2013}.
 We can optimize the continuous relaxation of the problem
 by the projected gradient algorithm,
 multiplicative algorithm,
 or interior point algorithm.
 From the obtained the optimal allocation $p$,
 we wish to design a discrete allocation for fixed sample size $t$.

 Given an allocation $p \in {\cal P}$,
 recall that ${\rm supp}(p)= \{j \in [K] \colon p_j >0  \}$.
 Let $t_i$ be the number of pulls for arm $i \in {\rm supp}(p)$
 and $s$ be the size of $ {\rm supp}(p)$.
Then, letting the {\em frequency } $t_i= \lceil \left(t- \frac{1}{2}s\right)p_i \rceil$
results in $\sum_{i \in {\rm supp}(p)}t_i$ samples.
If $\sum_{i \in {\rm supp}(p)}t_i=t$, this allocation is a desired solution.
Otherwise, we conduct the following procedure 
until the $\sum_{i \in {\rm supp}(p)}t_i -n$ is $0$;
increase a frequency $t_j$ which attains
$t_j/p_j =\min_{i \in {\rm supp}(p)}t_i/p_i$ to $t_j+1$,
or decreasing some $t_j$ with $(t_j-1)/p_j= \max_{i \in {\rm supp}(p)}(t_i-1)/p_i$ to $t_j-1$.
Then $(t_i, \ldots, t_s)$ lies in the efficient design apportionment
(see~\cite{pukelsheim2006}.)
Note that since the relaxation problem has exponential number of variables in our setting, we are restricted to the number of 
${\rm supp}(p)$
instead of dealing with all super-arms.


\begin{algorithm}[t]
	\caption{Exhaustive Search ({\tt Exhaustive})} \label{alg:fullsearch}
	\SetKwInOut{Input}{Input}
	\SetKwInOut{Output}{Output}
	\Input{Accuracy $\varepsilon>0$, confidence level $\delta \in (0,1)$,
	allocation strategy $p$}
    \For{$t =1, \ldots, n$}{
    Pull $M_t \in {\rm supp}(p)$;

    Observe $r_t$;
    
    Update $A_t$ and $b_t$;
    }
	\While{$Z_t^* \geq \varepsilon$}{
    $t \leftarrow t+1$;
    
    Pull $M_t \leftarrow \argmin_{M \in {\rm supp}(p)} \frac{T_M(t)}{p_M}$;

    Observe $r_t$;
    
    Update $A_t$ and $b_t$;
    
    $\widehat{M}^*_t   \leftarrow  {\rm argmax}_{M \in {\cal M}} \widehat{\theta}_t(M)$;
    
    $Z_t^* \leftarrow \max_{M\in \mathcal{M}\setminus \{\widehat{M}^*\}} \left(\widehat{\theta}_t(M)+C_t \normAxy \right)-\widehat{\theta}_t(\widehat{M}^*)$;
}
    \Return $\widehat{M}^* \leftarrow \widehat{M}^*_t  $\;
\end{algorithm}

\section{Details of Experiments}
\setcounter{equation}{0}
All experiments were conducted on a Macbook with a 1.3 GHz Intel Core i5 and 8GB memory. All codes were implemented by using Python.
The entire procedure of {\tt Exhaustive} is detailed in Algorithm~5.
This algorithm reduces our problem
to the pure exploration problem in the linear bandit,
and thus runs in exponential time,
i.e, $O(n^k)$.
In all experiments,
we employed the approximation algorithm called the \emph{greedy peeling}~\cite{Asahiro2000} as the D$k$S-Oracle.
Specifically, the greedy peeling algorithm iteratively removes a vertex with the minimum weighted degree in the currently remaining graph until we are left with the subset of vertices with size $k$.
The algorithm runs in $O(n^2)$.


\section{Proofs}
\setcounter{equation}{0}
First, we introduce the notation.
For $M, M' \in \cM$,
let $\Delta(M, M')$ be the \emph{value gap} between two super-arms, 
i.e., $\Delta(M, M')=|\theta(M)-\theta(M')|$. 
Also, let $\widehat{\Delta}(M, M')$ be the {\em empirical gap} between two super-arms, i.e., 
$\widehat{\Delta}(M, M')=|\thetahat(M)-\thetahat(M')|$. 

\subsection{Proof of Lemma~2}

\begin{proof}
The empirical best super-arm $\widehat{M}^*_t$ can be computed by the greedy algorithm under matroid constraint~\cite{Karger1998} (the greedy algorithm is described in Appendix~\ref{appendix:matroid}).
The maximization of $\mathrm{P}_1$ 
is linear maximization under matroid constraint,
and thus,
this is also solvable by the greedy algorithm. 
Letting $g(n)$ be the computation time for
checking whether a super-arm satisfies the matroid constraint or not, 
we see that the running time of the greedy procedure is $O(n\log n+ng(n))$.
Moreover, updating $\Atin$ needs $O(n^2)$ time.
Therefore, we have the lemma.
\end{proof}

\subsection{Proof of Lemma~3}

\begin{proof}

First we define random event ${\cal E}$ as follows: 
\begin{align}
{\cal E}=\left\{\forall t \in \{1,2,\ldots,\},\, \forall M, M' \in {\cal M},\, |\thetahat(M)-\thetahat(M') | \right. 
\left. \leq C_t \sum_{i=1}^n |\chiM(i)-\chiMprime(i)| \sqrt{\Atin(i,i)} \right\}. \notag
\end{align}
We notice that
random event $\cE$ implies that
the event that the confidence intervals of
all super-arm $M \in {\cal M}$
are valid at round $t$.
From Lemma~1,
we see that the probability that event ${\cal E} $ occurs is at least $1-\delta$.
Under the event $\cE$,
we see that the output $\widehat{M}^*$ is an $\varepsilon$-optimal super-arm.
In the rest of the proof,
we shall assume that event $\cE$ holds.
Next, we focus on bounding the sample complexity $T$.
By recalling the stopping 
condition~\eqref{stop_independet},
a sufficient condition for stopping is that for $M^*$ and for $t>n$,
\begin{align}\label{stop_suf_matroid}
 \varepsilon
> \max_{M \in {\cM \setminus \{M^*\}}} \left( \widehat{\theta}_t(M)+C_t \sum_{i=1}^n |\chiM(i)-\chiMstar(i)| \sqrt{\Atin(i,i)}  \right)-\widehat{\theta}_t(M^*).
\end{align}
Let $\overline{M} =\argmax_{M \in {\cM \setminus \{M^*\}}} \left( \widehat{\theta}_t(M)+C_t \sum_{i=1}^n |\chiM(i)-\chiMstar(i)| \sqrt{\Atin(i,i)}  \right)$.
Eq.~(\ref{stop_suf_matroid})
is satisfied if
\begin{align}\label{stopForOpt2}
\widehat{\Delta}(M^*, \overline{M}) 
> C_t \sqrt{\frac{\rho(p)}{t}} -\varepsilon.
\end{align}
From Lemma~1 with $x=\chiMstar-\bm{\chi}_{ \scalebox{0.5} {$\overline{M}$}}$, with probability at least  $1-\delta$,
we have
\begin{align}\label{eq:proposi2}
\widehat{\Delta}(M^*, \overline{M}) \geq \Delta(M^*,\overline{M})
-C_t \sum_{i=1}^n |\chiMstar(i)-\bm{\chi}_{ \scalebox{0.5} {$\overline{M}$}}(i)| \sqrt{\Atin(i,i)} 
\geq \Delta(M^*,\overline{M})
-C_t \sqrt{\frac{\rho(p)}{t}}.
\end{align}
Combining \eqref{stopForOpt2} and \eqref{eq:proposi2},
we see that
a sufficient condition for stopping is given by 
$\Delta(M^*, \overline{M}) \geq \Delta_{\min}\geq 2 C_t \sqrt{\frac{\rho(p)}{t}} -\varepsilon$.
Therefore, we have 
$
t \geq 4C_t^2H_{\varepsilon}
$
as a sufficient condition to stop.
Let $\tau >n$ be the stopping time of the algorithm.
From the above discussion,
we see that
$\tau \leq 4C_\tau^2 H_{\varepsilon}$.
Recalling that $C_t =\sigma\sqrt{2\log(c't
^2 n/\delta) }$,
we have
$\tau \leq 8\sigma^2  \log (c'\tau^2n/\delta)H_{\varepsilon}$.
Let $\tau'$ be a parameter that satisfies
\begin{align}\label{tauprime}
\tau= 8\sigma^2  \log (c'\tau'^2n/\delta)H_{\varepsilon}.
\end{align}
Then, it is obvious that $\tau' \leq \tau$ holds.
For $N$ defined as $N= 8\sigma^2  \log (c'n/\delta)H_{\varepsilon}$,
we have
\begin{align*}
\tau'  \leq \tau 
 = 16\sigma^2  \log (\tau')H_{\varepsilon}+N 
\leq 16 \sigma^2  \sqrt{\tau'} H_{\varepsilon}+N 
\end{align*}
Transforming this inequality,
we obtain
\begin{align}\label{ineq:l}
\sqrt{\tau'} 
\leq  8 \sigma^2 H_{\varepsilon}+\sqrt{64\sigma^4\He^2+N} 
\leq 2\sqrt{64\sigma^4\He^2+N}.
\end{align}
Let $L=2\sqrt{64\sigma^4\He^2+N}$,
which equals the RHS of \eqref{ineq:l}.
  We see that $\log L= O \left(
  \log \left( \sigma \He \left( \sigma^2 \He+\log\left( \frac{n}{\delta}\right) 
  \right)  \right) \right)$.
Then, using this upper bound of $\tau'$ in \eqref{tauprime}, 
we have
\[
\tau \leq  
16\sigma^2 \He \log\left (   \frac{c'n}{\delta} \right)+C(H_\varepsilon, \delta),
\]
where
\begin{align*}
C(\He, \delta)&= O\left( \sigma^2 \He  \log \left( \sigma \He \left( \sigma^2\He +\log\left( \frac{n}{\delta}\right)   \right) \right) \right).
\end{align*}
Recalling that $\sigma=kR$, we obtain
\begin{align*}
\tau =
O\left(
k^2R^2\He \log \left( 
\frac{n}{\delta} 
\left(
kR\He
\left(k^2R^2\He +\log\left( \frac{n}{\delta}  \right)\right)
\right)
\right)
\right).
\end{align*}

\end{proof}


\subsection{Proof of Theorem~1}

We begin by showing the following three lemmas.

\begin{lemma}\label{lemma:positive}
Let $W \in \mathbb{R}^{n \times n}$ be any positive definite matrix.
Then $\widetilde{G}=(V,E, \widetilde{w})$ constructed by Algorithm~1 is an non-negative weighted graph.
\end{lemma}
\begin{proof}
For any $(i,j) \in V^2$,
we have $w_{ii} \geq 0$  and $w_{jj} \geq 0$ since $W$ is a positive definite matrix.
If $w_{ij} \geq 0$, it is obvious that $\tilde{w}_{ij}=w_{ij}+w_{ii}+w_{jj}\geq 0$.
We consider the case $w_{ij} <0$.
In the case, we have $w_{ij}+w_{ii}+w_{jj} > 2w_{ij}+w_{ii}+w_{jj} \geq  0$, where the last inequality holds from the definition of positive definite matrix $W$.
Thus, we obtain the desired result.

\end{proof}

\begin{lemma}\label{lemma:wtilde}
Let $W \in \mathbb{R}^{n \times n}$ be any positive definite matrix
and $\widetilde{W}=(\widetilde{w}_{ij})$ be the adjacency matrix of
the complete graph constructed
by Algorithm~1. 
Then, for any $S \subseteq V$ such that $|S| \geq 2$, we have $w(S) \leq \widetilde{w}(S)$.
\end{lemma}

\begin{proof}
We have 
\begin{alignat*}{4}
w(S)=\sum_{e\in E(S)} w_e & =\sum_{\{i,j\}\in E(S)\colon i\neq j}w_{ij} +\sum_{i \in S }w_{ii} \\
& \leq \sum_{\{i, j\} \in E(S)\colon i \neq j}w_{ij} + (|S|-1)\sum_{i \in S}w_{ii} =\widetilde{w}(S),
\end{alignat*}
where the last inequality holds since each diagonal component $w_{ii}$ is positive for all $i \in V$ from the definition of the positive definite matrix.
\end{proof}

\begin{lemma}\label{lemma:ratio}
Let $W \in \mathbb{R}^{n \times n}$ be any positive definite matrix 
and $\widetilde{W}=(\tilde{w}_{ij})$ be the adjacency matrix of
the complete graph constructed
in Algorithm~1. 
Then, for any subset of vertices $S \subseteq V$, 
we have $\frac{\widetilde{w}(S)} {w(S)}
\leq  (|S|-1) \frac{\lambda_{\max}(W)  }{\lambda_{\min}(W)}$,
where $\lambda_{\min}(W)$ and $\lambda_{\max}(W)$ represent the minimum and maximum eigenvalues of $W$, respectively.
\end{lemma}
\begin{proof}
We consider the following two cases: 
Case (i) $\sum_{\{i, j\} \in E(S) \colon i \neq j}w_{ij} \geq 0 $
and Case (ii) $\sum_{\{i, j\} \in E(S)\colon i \neq j}w_{ij} < 0 $.

\paragraph{Case (i)}
Since $W=(w_{ij})_{1 \leq i, j \leq n}$ is positive definite matrix,
we see that
diagonal component $w_{ii}$ is positive for all $i \in V$.
Thus,
we have
    \begin{alignat*}{4}
      \widetilde{w}(S) &=\sum_{(i, j) \in E(S)\colon i \neq j}w_{ij}
    + (|S|-1)\sum_{i \in S}w_{ii} \\
   & \leq (|S|-1) \left(  \sum_{(i, j)
    \in E(S)\colon i \neq j}w_{ij}
    + \sum_{i \in S}w_{ii}  \right) =(|S|-1)w(S).  
    \end{alignat*}

Since $W$ is positive definite, 
we have $w(S)>0$. That gives us the desired result.

\paragraph{Case (ii)}
  In this case, we see that
    \begin{align*}
     \widetilde{w}(S) =\sum_{(i, j) \in E(S)\colon i \neq j}w_{ij} + (|S|-1)\sum_{i \in S}w_{ii} 
     \leq  (|S|-1)\sum_{i \in S}w_{ii}.
    \end{align*}

    For any diagonal component $w_{ii}$ we have that $w_{ii} \leq \max_{1 \leq i, j \leq n}  w_{ij} $.
    For the largest component  $\max_{1 \leq i, j \leq n}  w_{ij}$, we have
    \begin{alignat*}{3}
    \max_{1 \leq i, j \leq n}w_{ij} & 
    \leq \max_{1 \leq i, j \leq n}  \frac{1}{2} e_i^{\top} W e_i +\frac{1}{2} e_j^{\top} W e_j 
    & \leq \lambda_{\max}(W),
    \end{alignat*}
    where the first inequaltiy is satisfied since $W$ is positive definite.
    Thus, we obtain
    \begin{align} \label{tildewS}
    \widetilde{w}(S) < |S|(|S|-1)\lambda_{\max}(W).
    \end{align}
    For the lower bound of $w(S)$, we have
  \begin{align}\label{wS}
  w(S)= \bm{\chi}_{\scalebox{0.5}{$S$}}^\top W \bm{\chi}_{\scalebox{0.5}{$S$}} 
  =\frac{ \bm{\chi}_{ \scalebox{0.5} {$S$}}^\top W \bm{\chi}_{ \scalebox{0.5} {$S$}} }{   \|\bm{\chi}_{ \scalebox{0.5} {$S$}}^\top \bm{\chi}_{ \scalebox{0.5} {$S$}} \|_2^2 }|S|>\lambda_{\min}(W)|S| .
  \end{align}
   
  Combining (\ref{tildewS}) and (\ref{wS}), we obtain
  \begin{align*}
  \frac{\widetilde{w}(S)}{w(S)} < \left(|S|-1\right) \frac{\lambda_{\max}(W) }{\lambda_{\min}(W)}, 
  \end{align*}
  which completes the proof. 

\end{proof}
We are now ready to prove Theorem~1. 
\begin{proof}[Proof of Theorem~1]
For any round $t>n$,
let $\bm{\chi}_{ \scalebox{0.5} {$S_t$}}$ be the approximate solution obtained by Algorithm~1 and $Z$ be its objective value.
Let $\widetilde{w}: 2^{[n]}\rightarrow \mathbb{R}$ be the weight function defined by Algorithm~1.
We denote the optimal value of QP by OPT.
Let us denote optimal solution of the D$k$S for $G([n],E, \widetilde{w})$ by $\widetilde{S}_\text{OPT}$.
Adjacency matrix $W$ is a symmetric positive definite matrix;
thus, Lemmas~4, 5 and 6 hold for $W$.
We have 
\begin{alignat}{5}\label{eq:tildew}
\widetilde{w}(S_t)& \geq \alpha_\text{D$k$S}\widetilde{w}(\widetilde{S}_\text{OPT}) \hspace{1.0cm} (\because S_t \ \text{is an} \ \alpha_\text{D$k$S}\text{-approximate solution for D$k$S}(\widetilde{G}). )\notag\\ \notag
&\geq \alpha_\text{D$k$S}\widetilde{w}(S_\text{OPT}) \\ 
&\geq \alpha_\text{D$k$S}w(S_\text{OPT}). \hspace{1.0cm}  (\because {\rm Lemma~5}).
\end{alignat}
Thus, we obtain 
\begin{alignat*}{5}
 Z=\bm{\chi}_{ \scalebox{0.5} {$S_t$}}^\top \Atin \bm{\chi}_{ \scalebox{0.5} {$S_t$}}& = w(S_t)   \notag  \\\notag 
 &\geq \frac{1 }{|S|-1} \frac{\lambda_{\min}(W) }{\lambda_{\max}(W)} \tilde{w}(S_t) \hspace{0.3cm} (\because {\rm Lemma~6}) \\\notag
&  \geq \frac{1}{|S|-1}  \frac{\lambda_{\min}(W) }{\lambda_{\max}(W)}  \alpha_\text{D$k$S}w(S_\text{OPT}) \hspace{0.3cm} (\because {\rm \eqref{eq:tildew}})\\
&  = \frac{1}{|S|-1}  \frac{\lambda_{\min}(W) }{\lambda_{\max}(W)}  \alpha_\text{D$k$S}\text{OPT}.
\end{alignat*}

Therefore, we obtain
 $ Z \geq  \left(
 \frac{1}{k-1} 
\frac{\lambda_{\min}(W)}{ \lambda_{\max}(W)} \alpha_\text{D$k$S}
\right)
\text{OPT}$. 
\end{proof}


\subsection{Proof of Theorem~2}
\begin{proof}

Updating $\Atin$ can be done in $O(n^2)$ time,
and computing the empirical best super-arm can be done in $O(n)$ time.
Moreover,
confidence maximization CEM can be approximately solved in polynomial-time,
since quadratic maximization QP is solved in polynomial-time as long as we employ polynomial-time algorithm as the D$k$S-Oracle.
Let ${\rm poly}(n)_{{\rm D}k {\rm S} }$ be the computation time of D$k$S-Oracle.
Then, we can guarantee that {\tt SAQM} runs in $O( \max\{ n^2, {\rm poly}(n)_{{\rm D}k {\rm S} } \})$ time.
Most existing approximation algorithms for the D$k$S have efficient computation time.
For example,
if we employ the algorithm by
Feige, Peleg, and Kortsarz~\cite{Feige2001}
as the D$k$S-Oracle that runs in $O(n^\omega)$ time in Algorithm~1, 
the running time of \texttt{SAQM} becomes $O(n^{\omega})$,
where the exponent $\omega\le 2.373$ is equal to that of the computation time of matrix multiplication
(e.g., see 
\cite{LeGall14}).
If we employ the algorithm by
Asahiro et al.~\cite{Asahiro2000} that runs in $O(n^2)$,
the running  time of \texttt{SAQM} also becomes $O(n^2)$.
\end{proof}

\subsection{Proof of Theorem~3}\label{apx:thm6}

Before stating the proof of Theorem~3,
we give the two technical lemmas.

\begin{lemma}\label{lemma:condition_number}
For any round $t$,
the condition number of $A_{\bf{x}_t}$ is bound by
\begin{align*}
\frac{\lambda_{\max}(\At)}{\lambda_{\min}(\At)}=O(k).
\end{align*}

\end{lemma}
\begin{proof}
For $\lambda_{\max}(\At)$, we have
\begin{align*}
\lambda_{\max}(\At)&= \max_{\|y\|_2=1} y^{\top}\Atin y 
=\max_{\|y\|_2=1} y^{\top}\sum_{t'=1}^t x_{t'}^{\top} x_{t'}y 
=\max_{\|y\|_2=1} \sum_{t'=1}^t  \left( x_{t'}^{\top} y\right)^2 
 \leq kt.
\end{align*}
Next, we give a lower bound of  $\lambda_{\min}(\At)$.
Recall that the sequence ${\bf x}_t= (\bm{\chi}_{\sm M_1},\ldots,\bm{\chi}_{\sm M_t})$ represents for
the sequence of $t$-set selections  ${\bf M}_t=(M_1, M_2, \ldots, M_t) \in \cM^t$ and $T_M(t)$ is the number of times
that super-arm $M$ is selected before $t+1$-th round.
In any super-arm selection strategy
that samples $M_t$ for any $t \in [T]$
such that
$\min_{M \in {\bf M}_t}  T_{M}(t) / t \geq r$ for some constant $r>0$,
we have $\lambda_{\min}(\At) \geq r \lambda_{\min}(\sum_{M \in {\bf M}_t} \chiM \chiM^{\top} )t$.
From the above discussion,
we have $\frac{\lambda_{\max}(\At)}{\lambda_{\min}(\At)}=O(k)$.

\end{proof}

Next,
for any $t>0$,
let us define random event ${\cal E}'_t$ as
\begin{align}\label{event}
\left\{ \forall M \in {\cal M},\ |\theta(M)-\thetahat(M) |  \leq C_t \| \bchi_{\sm {$M$}} \|_{\Atin}  \right\}.
\end{align}
We note that random event $\cE'_t$ characterizes
the event that the confidence bounds of
all super-arm $M \in {\cal M}$
are valid at round $t$.
Next lemma indicates that,
if the confidence bounds are valid,
then {\tt SAQM} always outputs $\varepsilon$-optimal super-arm $\Mhatstar$ when it stops.
\begin{lemma}\label{epsilonoptimal}
Given any $t>n$,
assume that ${\cal E}'_t$ occurs.
Then,
if {\tt SAQM} (Algorithm~2) terminates at round $t$,
we have $\theta(M^*)-\theta(\widehat{M}^*)\leq \varepsilon$.
\end{lemma}

\begin{proof}[Proof of Lemma~\ref{epsilonoptimal}]
If $\widehat{M}^*=M^*$, 
we have the desired result.
Then, we shall assume $\widehat{M}^*\neq M^*$.
We have the following inequalities:
\begin{alignat*}{5}
\theta(\Mhatstar) & \geq \thetahat(\Mhatstar_t)-C_t \normAhatstar  \hspace{0.8cm}({\rm \because event \ {\cal E}_t'}) \\ \notag
&\geq 
\max_{M\in {\cal M}\setminus \{\widehat{M}^*\}}
\widehat{\theta}(M)
+\frac{1}{\alpha_t}{Z_{t}}-\varepsilon  \hspace{0.8cm}  ({\rm \because stopping \ condition })  \\ \notag
&\geq  \max_{M\in {\cal M}
\setminus \{\widehat{M}^*\}}\thetahat(M)  +\maxM C_t \normA -\varepsilon  \hspace{0.8cm}\left(\because Z_t \geq \alpha_t  {\max}_{M \in {\cal M}}C_t \normA \right)  \\\notag
& \geq  \thetahat(M^*) +C_t\normAstar-\varepsilon \\\notag
&\geq  \theta(M^*)-\varepsilon.  \hspace{0.8cm}({\rm \because event \ {\cal E}'_t }) 
\end{alignat*}
\end{proof}

We are now ready to prove Theorem~3.

\begin{proof}[Proof of Theorem~3]
We define event $\cE'$ as $\bigcap_{t=1}^{\infty} \cE'_t$.
We can see that
the probability that event ${\cal E}' $ occurs is at least $1-\delta$ from Proposition~1.
In the rest of the proof,
we shall assume that this event holds.
By Lemma~\ref{epsilonoptimal}
and the assumption on $\cE'$,
we see that the output $\widehat{M}^*$ is $\varepsilon$-optimal super-arm.
Next, we focus on bounding the sample complexity.

A sufficient condition for stopping is that for $M^*$ and for $t>n$,
\begin{align}\label{stopForOpt_ecb}
 \widehat{\theta}(M^*)-C_t \normAstar
\geq  \max_{M  \in {\cal M}\setminus \{M^*\}}
\widehat{\theta}( M)+
\frac{1}{\alpha_t}{Z_{t}}-\varepsilon.
\end{align}
From the definition of $\Atin$,
we have $\Lambda_{p}=\frac{\At}{t}$.
Using $\normAstar \leq \maxM \normA$ and $Z_t \leq \maxM \normA $,
a sufficient condition for \eqref{stopForOpt_ecb} is equivalent to:

\begin{align}\label{eq:stopForOpt}
\widehat{\Delta}(M^*, \overline{M}) 
\geq \left(1+\frac{1}{\alpha_t}\right)C_t \sqrt{   \frac{\rho(p)}{t  }}-\varepsilon,
\end{align}
where $\overline{M} =\argmax_{M \in {\cal M}} \widehat{\Delta}_t (M^*, M)$. 
On the other hand, we have
\begin{align*}
\| \chiMstar-\bm{\chi}_{ \scalebox{0.5} {$\overline{M
}$}} \|_{\Atin} \leq 2 \maxM \|  \chiM \|_{\Atin} \leq 2 \sqrt{\frac{\rho(p)}{t}}.
\end{align*}
Therefore, from Proposition~1 with $\bm{x}=\chiMstar-\chiMbar$, with at least probability $1-\delta$,
we have
\begin{align}\label{proposi-condition}
\widehat{\Delta}_t(M^*,\overline{M}) \geq \Delta(M^*,\overline{M})
-C_t \| \chiMstar-\bm{\chi}_{ \scalebox{0.5} {$\overline{M}$}} \|_{\Atin}
\geq \Delta(M^*,\overline{M})
-2C_t \sqrt{\frac{\rho(p)}{t}}.
\end{align}

Combining \eqref{eq:stopForOpt} and \eqref{proposi-condition},
we see that 
a sufficient condition for stopping becomes the following inequality.
\begin{align*}
\Delta_{\min}-2C_t \sqrt{\frac{\rho(p)}{t}} \geq \left(1+\frac{1}{\alpha_t}\right)C_t \sqrt{\frac{\rho(p)}{t}}-\varepsilon.
\end{align*}
Therefore, we have that a sufficient condition to stop is 
$
t \geq \left(3+\frac{1}{\alpha_t}\right)^2 C_t^2H_{\varepsilon},$
where 
$H_\varepsilon=  \frac{\rho(p)}{(\Delta_{\min}+\varepsilon)^2}$. 
Let $\tau >n$ be the stopping time of the algorithm.
From the above discussion,
we see that
$\tau \leq \left(3+\frac{1}{\alpha_\tau}\right)^2C_\tau^2 H_{\varepsilon}.$
Recalling that $C_\tau= c\sqrt{ \log (c'\tau^2K/\delta)}$,
we have that
\begin{align*}
\tau \leq \left(3+1/\alpha_\tau\right)^2C_\tau^2 H_{\varepsilon} =\left(3+1/\alpha_\tau\right)^2 c^2  \log (c'\tau^2K/\delta)H_{\varepsilon}.
\end{align*}
Let $\tau'$ be a parameter that satisfies
\begin{align}\label{tau}
\tau=\left(3+1/\alpha_\tau\right)^2 c^2  \log (c'{\tau'}^2K/\delta)H_{\varepsilon}.
\end{align}
Then, it is obvious that $\tau' \leq \tau$ holds.
For $N$ defined as
$N=\left(3+1/\alpha_\tau \right)^2 c^2  \log (c'K/\delta)H_{\varepsilon}$,
we have
\begin{align*}
\tau'  \leq \tau 
 = \left(3+1/\alpha_\tau \right)^2 c^2  \log ({\tau'})H_{\varepsilon}+N 
\leq \left(3+1/\alpha_\tau \right)^2 c^2  \sqrt{\tau'} H_{\varepsilon}+N.
\end{align*}
By solving this inequality with $c=2\sqrt{2}\sigma$, we obtain
\begin{alignat*}{5}
\sqrt{\tau'} 
&\leq  4\left(3+1/\alpha_\tau \right)^2 \sigma^2 H_{\varepsilon}+\sqrt{16\left(3+1/\alpha_\tau \right)^4\sigma^4\He^2+N} \\
&\leq 2\sqrt{16\left(3+1/\alpha_\tau \right)^4\sigma^4\He^2+N}.
\end{alignat*}
Let $L=2\sqrt{16\left(3+1/\alpha_\tau \right)^4\sigma^4\He^2+N}$,
which is equal to
the RHS of the inequality.
  We see that $\log L= O\left(
  \log \left(  \frac{\sigma}{\alpha_{\tau}} \He \left(   (\frac{\sigma}{\alpha_\tau})^2\He +\log\left( \frac{K}{\delta}\right) \right)  \right)\right)$.
Then, using this upper bound of $\tau'$ in \eqref{tau},
we have
\[
\tau \leq  8\left(3+\frac{1}{\alpha_\tau}\right)^2 \sigma^2 \He \log\left (   \frac{c'K}{\delta} \right)+C(H_\varepsilon, \delta),
\]
where 
\begin{align*}
C(\He, \delta)&=16\left(3+\frac{1}{\alpha_\tau}\right)^2 \sigma^2 \He \log(L) \\
&=O\left(\sigma^2 \He \log \left(  \frac{\sigma}{\alpha_\tau} \He \left(   \frac{\sigma^2}{\alpha_\tau^2}\He +\log\left( \frac{K}{\delta}\right) \right)  \right)\right)\notag.
\end{align*}

We see that $1/\alpha_{\tau} = O\left( kn^{1/8} \right)$ from Theorem~1 and Lemma~7,
if we use the best approximation algorithm for the D$k$S as the D$k$S-Oracle~\cite{Bhaskara2010}.
Recalling that $\sigma=kR$ and $\log(K/ \delta)\leq k\log(n/\delta)$, 
we obtain 
\begin{align*}
\tau =
O\left(
k^2R^2\He \left(
n^{\frac{1}{4}}k^3 \log \left( \frac{n}{\delta}\right)
+ \log \left( 
n^{\frac{1}{8}} k^3R \He 
\left(
n^{\frac{1}{4}} k^3R^2\He+
\log\left( \frac{n}{\delta}  \right)
\right)
\right)
\right)
\right).
\end{align*}
\end{proof}





\end{document}